\documentclass[letterpaper]{article} 
\usepackage{aaai23-r2hcai}  
\usepackage{times}  
\usepackage{helvet}  
\usepackage{courier}  
\usepackage[hyphens]{url}  
\usepackage{graphicx} 
\usepackage{natbib}  
\usepackage{caption} 
\usepackage{algorithm}
\usepackage{algorithmic}

\usepackage{newfloat}
\usepackage{listings}
\DeclareCaptionStyle{ruled}{labelfont=normalfont,labelsep=colon,strut=off} 
\floatstyle{ruled}
\newfloat{listing}{tb}{lst}{}
\floatname{listing}{Listing}
\usepackage{amsmath}
\usepackage{amssymb}
\usepackage{enumitem}

\def\RR{\mathbb{R}}

\def\OO{\mathcal{O}}

\def\TT{\mathcal{T}}

\def\BB{\mathcal{B}}

\usepackage{amsthm}
\newtheorem{theorem}{Theorem}
\newtheorem{defn}{Definition}
\newtheorem{prin}{Principle}
\newtheorem{conj}{Conjecture}

\usepackage{fancyhdr}
\fancypagestyle{firstpagehf}
{
   \fancyhf{}
   \fancyhead[HC]{The AAAI 2023 Workshop on Representation Learning for Responsible Human-Centric AI (R$^2$HCAI)}
   \fancyfoot[C]{\thepage}
}

\title{The Shape of Explanations: A Topological Account of Rule-Based Explanations in Machine Learning}
\author{
    Brett Mullins$^1$
}
\affiliations {
    \textsuperscript{\rm 1} University of Massachusetts Amherst \\
    bmullins@umass.edu
}

\usepackage{bibentry}

\begin{document}
\thispagestyle{firstpagehf}
\maketitle

\begin{abstract}
	Rule-based explanations provide simple reasons explaining the behavior of machine learning classifiers at given points in the feature space. Several recent methods (Anchors, LORE, etc.) purport to generate rule-based explanations for arbitrary or black-box classifiers. But what makes these methods work in general? We introduce a topological framework for rule-based explanation methods and provide a characterization of explainability in terms of the definability of a classifier relative to an explanation scheme. We employ this framework to consider various explanation schemes and argue that the preferred scheme depends on how much the user knows about the domain and the probability measure over the feature space.
\end{abstract}

\section{Introduction}

Explanations for predictions of machine learning models act as reasons for a predictive model's behavior and are desirable for trustworthy and transparent machine learning \cite{Mo18}. With this being said, there is not much agreement in the machine learning community on exactly what counts as an explanation \cite{Do17, burkart2021survey}. Machine learning practitioners have developed a large set of domain-specific explanation methods in recent years. These methods are tailored either to the specific task the model is seeking to perform, e.g. regression, classification, object detection, etc., or to the type of inputs and outputs of the model, e.g., tabular features, images, sentences, etc. \cite{islam2022systematic}.

A promising technique for explaining the predictions of structured or tabular classifiers is rule-based explanations. A rule-based explanation is a predicate defining a simple region in the feature space that is sufficient for classifying a given point. In this paper, we take advantage of the connection between the inherent definability of rule-based explanations and definability in topology to develop a general framework to represent varieties of explanations based on existing explanation algorithms.

To summarize this paper, we make the following contributions:
\begin{itemize}
  \item We present a novel framework of explainability for rule-based classifiers based on existing explanation algorithms.
  \item We characterize explainability as a topological property relative to an explanation scheme i.e. relative to a choice of explanation shape and a measure of explanation size. We conjecture that all classifiers ``in the wild'' satisfy this notion of explainability.
  \item Employing our framework, we identify two principles for explanation algorithms that apply both theoretically and in practice. The first is that rule-based explanations can take nearly any desired shape. The second holds that if no probability measure is known over the feature space and at least one feature is not bounded, then explanations should be bounded, i.e. include all unbounded features.
\end{itemize}

This paper proceeds as follows. In Section \ref{s:background}, we discuss various existing explanation algorithms and provide a brief introduction to topology. We introduce explanation schemes as a framework for explainability and characterize explainability as a topological property in Sections \ref{s:explanation_schemes}, \ref{s:evd}, respectively. In Section \ref{s:choosing}, we derive principles for both formal and practical explanation algorithms. In Sections \ref{s:open_probs}, \ref{s:related_work}, we conclude by discussing limitations and open problems and surveying related work.

\section{Background} \label{s:background}

\subsection{Rule-Based Explanations}

In this section, we introduce rule-based explanation algorithms and consider their representative properties. Given a classifier and a point in the feature space, a rule-based explanation algorithm generates a rule defined in terms of the features of the classifier that both covers the given point and is sufficient for its classification. Rule-based explanations are perturbation resistant in the sense that the explanation applies to a neighborhood about a given point. Additionally, this sort of explanation is often called post-hoc, since it occurs after the model is constructed, and local, since the explanation is specific to the given point \cite{Gu18}. Let us consider four such explanation algorithms: Anchors \cite{Ri18}, PALEX \cite{jia2020exploiting}, LoRE \cite{Gu18a}, and LoRMIkA \cite{rajapaksha2020lormika}.

Given black-box access to a classifier and a point, these algorithms evaluate the classifier on a collection of points sampled in a neighborhood of the given point and return a rule such that the points satisfying the rule (or at least some proportion of points above a given threshold) evaluate to the same label as the given point. Anchors and PALEX each return rules that are the conjunction of predicates involving individual features which only use $<, \ =, \ >$ relations and $\land$ connective. By contrast, LoRE and LoRMIkA return a rule defined by a simple decision tree as well as counterfactuals, though we ignore the latter in the present analysis. Notice that all four algorithms return rules that define rectangles in the feature space.

Rule-based explanation algorithms often claim to be model agnostic i.e. the explanation algorithm makes no assumption on the given classifier's functional form to generate post-hoc explanations \cite{ignatiev2019abduction}. This suggests that such algorithms will generate explanations for any functional form.

\subsection{Topology} \label{s:topology}

Let $X$ be a set. A \emph{topology} $\TT$ is a collection of subsets of $X$ with the following properties:
\begin{enumerate} 
  \item $X, \emptyset \in \TT$;
  \item $\TT$ is closed under arbitrary union;
  \item $\TT$ is closed under finite intersection.
\end{enumerate}

\noindent We refer to $(X, \TT)$ as a \emph{topological space}. In a topological space $(X, \TT)$, a subset $\OO \subseteq X$ is called \emph{open} in $X$ if $\OO \in \TT$. A subset $C \subseteq X$ is called \emph{closed} in $X$ if $X \setminus C \in \TT$.
The \emph{interior} of a set $A \subseteq X$, $\textup{Int}(A)$, is the largest open set contained in $A$, and the \emph{closure} of a $A$, $\overline{A}$, is the smallest closed set containing $A$.

A \emph{basis} $\BB$ for $\TT$ is a collection of open sets of $\TT$ such that every open set of $\TT$ is the union of elements of $\BB$. We refer to sets in the basis as \emph{basic open sets}, and, if $\BB$ is a basis for $\TT$, then we say that $\TT$ is generated by $\BB$. If $A \in \TT$ and $\TT$ has basis $\BB$, then we say that $A$ can be defined in terms of basic open sets from $\BB$.

We say that a set $A \subseteq X$ is \emph{dense} in $X$ if $\OO \cap A \neq \emptyset$ for every non-empty open set $\OO \subseteq X$. A set $B \subseteq X$ is said to be \emph{nowhere dense} in $X$ if $\textup{Int}(\overline{B}) = \emptyset$. Nowhere dense sets fail to cover any part of $X$ with respect to $\TT$.
A set $C \subseteq X$ is called \emph{meagre} if $C$ is the countable union of nowhere dense sets. If $C$ is meagre with respect to topology $\TT$, then $C$ is said to be $\TT$-meagre. The notion of meagre is what we mean by topologically small.

\section{Explanation Schemes} \label{s:explanation_schemes}

Let $f: X \rightarrow Y$ be a classifier. A rule-based explanation for $x \in X$ is a well-defined region of the feature space containing $x$ whose classification is invariant within the region, i.e. belonging to the region is sufficient to be classified as $f(x)$. The intuition of a rule-based explanation is that the label assigned by the classifier is unaffected by perturbations so long as the perturbed point remains within the region \cite{Gu18a}. What properties, then, do these regions have, and what shape do they take? In this section, we develop a general topological framework to represent the properties of rule-based explanations. We explore how to represent explainability within this framework in Section \ref{s:evd} below.

Rule-based explanations are definable subsets of the feature space that belong to a common class, i.e. they satisfy some predicate or definable property $\varphi$. For instance, $\varphi$ may be the predicate open rectangle or open ball. More generally, the candidate subsets for explanations belong to the class of subsets of the feature space satisfying $\varphi$: $\{ A \subseteq X \ | \ \varphi(A) \}$. Not all predicates, however, are suitable to be explanations. Let us restrict attention to the following set of rules.

\begin{defn} \label{def:scalable_rule}
  A scalable rule $\varphi$ for $X$ is a predicate $\varphi$ such that $X_{\varphi} = \{ A \subseteq X \ | \ \varphi(A) \}$ satisfies the following two conditions:
  \begin{enumerate} 
    \item $\bigcup X_{\varphi} = X$ \label{cond:1}
    \item If $x \in A_1 \cap A_2$ for $A_1, A_2 \in X_{\varphi}$, there exists $A_3 \in X_{\varphi}$ such that $x \in A_3$ and $A_3 \subseteq A_1 \cap A_2$. \label{cond:2}
  \end{enumerate}
\end{defn}

\noindent Condition \ref{cond:1} says that each point in the feature space is covered by a rule i.e. a potential explanation. Condition \ref{cond:2} says that rules can be defined as small as needed. While not all scalable rules make for desirable explanations for the user, we hold the converse is true. Scalable rules provide the necessary minimal structure for this analysis.

Observe that the collection of subsets satisfying a scalable rule $\varphi$ meets the conditions for being a topological basis \cite{Mu00}. By closing this collection under the operations of countable union and finite intersection, we obtain the topology $\TT_{\varphi}$ We say that $\TT_{\varphi}$ is the \emph{explanation topology} generated by scalable rule $\varphi$. The explanation topology consists of sets that can be defined in terms of rules of the form $\varphi$.

Though there may be many rules covering a given point in the feature space, we assign a notion of size or robustness to a rule called coverage \cite{Ri18}. We define coverage as a measure $\mu$ in the measure-theoretic sense. Typically, if a probability measure $p$ is known over the feature space, then coverage of a given rule $A$ is the measure of $A$ with respect to $p$. Similarly, if a probability measure is unknown, then the counting measure is typical for discrete feature spaces, Lebesgue measure is typical for continuous feature spaces, and a product measure is typical for spaces with both discrete and continuous features. If an explanation has zero coverage with respect to $\mu$, we say that explanation is $\mu$-null. Usually, the user will prefer an explanation with greater coverage to an explanation with lesser coverage; however, other factors may influence a user's preferences. For instance, a user may prefer explanations that involve fewer features to more features for a given amount of coverage \cite{Mo18}.

Let us conclude by packaging together the terms introduced in this section:

\begin{defn}
  An explanation scheme is a tuple $(X, \varphi, \mu)$ where $X$ is the feature space, $\varphi$ is a scalable rule generating the explanation topology $\TT_{\varphi}$ on $X$, and $\mu$ is a measure on $X$ representing coverage.
\end{defn}

\section{Explainability via Definability} \label{s:evd}

In this section, we introduce a notion of explainability relative to an explanation scheme and demonstrate that explainability is equivalent to a simple topological property. In particular, we find that a classifier is explainable if the preimage of each label is the union of a low complexity set, i.e. an open set, and a set of points that is small with respect to both the scalable rule and coverage measure. Let us first introduce our notion of explainability.

\begin{defn} \label{def:explainable}
  A classifier $f: X \rightarrow Y$ is explainable for scheme $(X, \varphi, \mu)$ if each $x \in X$ has an explanation except on a set of edge cases.
\end{defn}

Edge cases are typically taken to be a small number of instances in which some desired property does not hold. For instance, in the present setting, the prototypical edge case is a point on the decision boundary of a classifier for a continuous feature space. We formalize the notion of smallness as topologically small with respect to rules of the form $\varphi$ and small in coverage with respect to measure $\mu$.

\begin{defn}
	A set $E$ is a set of edge cases for scheme $(X, \varphi, \mu)$ if $E$ is $\TT_\varphi$-meagre and $\mu$-null.
\end{defn}

Being meagre and null are necessary for a set of edge cases since neither alone implies smallness in the sense of our prototypical edge case set. For instance, given the standard topology, it is known that $\RR$ can be partitioned into a meagre set and a Lebesgue null set \cite{oxtoby1980measure}.
This identification between the notion of edge cases and topologically and measure-theoretically small sets leads to our main result:

\begin{theorem} \label{th:characterization}
  A classifier $f: X \rightarrow Y$ is explainable for scheme $(X, \varphi, \mu)$ if and only if, for $y \in Y$, there exists open set $\OO_y \in \TT_\varphi$ such that $f^{-1}(y) = \OO_y \cup E_y$ and $E_y$ is $\TT_\varphi$-meagre, $\mu$-null.
\end{theorem}

\begin{proof}
  $(\rightarrow)$. Suppose $f: X \rightarrow Y$ is explainable for $(X, \varphi, \mu)$. Let $\OO \subseteq X$ be the set of points in the feature space with explanations and $E = X \setminus \OO$. Then $E$ is $\TT_{\varphi}$-meagre, $\mu$-null. Let $y \in Y$. Define $\OO_y = f^{-1}(y) \cap \OO$ and $E_y = f^{-1}(y) \cap E$. Since meagre and null sets are closed under subset, $E_y$ is $\TT_{\varphi}$-meagre, $\mu$-null.
  Then, for $x \in \OO_y$, there is explanation $A_x$ covering $x$ such that $A_x \subseteq \OO_y$. Since $\OO_y = \cup_{x \in \OO_y} A_x$, $\OO_y$ is open in $\TT_{\varphi}$. \\
  $(\leftarrow)$. Suppose, for $y \in Y$, $f^{-1}(y) = \OO_y \cup E_y$ where $\OO_y$ is open for $\TT_{\varphi}$ and $E_y$ is $\TT_{\varphi}$-meagre, $\mu$-null. Let $x \in X$, and, for some $y \in Y$, $f(x) = y$. Then $x \in \OO_y \cup E_y$. WLOG, suppose $x \in \OO_y$.
  Then there exists basic open set $A_x \subseteq \OO_y$ covering $x$ such that $A_x$ satisfies $\varphi$. So $A_x$ is an explanation for $x$.
\end{proof}

Though Theorem \ref{th:characterization} is a simple characterization, this result allows us to determine explainability by only considering the geometry of the classifier with respect to a set of rules. As an example, consider a linear classifier on $n$ continuous features with rules of the form of open squares. The explanation topology generated from open squares is the standard Euclidean topology on $\RR^n$. In this topology, the preimage of one label is an open halfspace and the other is a closed halfspace which satisfies the condition for Theorem \ref{th:characterization}. So the linear classifier is explainable relative to this explanation scheme.

As a general application of Theorem \ref{th:characterization}, let us consider whether or not explainability is preserved by a voting ensemble classifier. We define a voting ensemble as follows \cite{dietterich2000ensemble}:

\begin{defn}
  If $f_1, \ldots, f_k$ are classifiers such that $f_i: X \rightarrow Y$, $1 \leq i \leq k$, then a voting ensemble of $f_1, \ldots, f_k$ is a classifier $f: X \rightarrow Y$ given by $f(x) = g(f_1(x), \ldots, f_k(x))$ where $g: Y^k \rightarrow Y$ where $g$ returns the most common label.
\end{defn}

\noindent For instance, a random forest is a common voting ensemble \cite{breiman2001random}. Below, we prove that voting ensembles preserve explainability with respect to a given explanation scheme by appealing only to topological properties.

\begin{theorem} \label{th:ensemble}
  If $f_1, \ldots, f_k$ are classifiers explainable for explanation scheme $(X, \varphi, \mu)$ and $f$ is the voting ensemble of $f_1, \ldots, f_k$, then $f$ is explainable for $(X, \varphi, \mu)$.
\end{theorem}

\begin{proof}
  Let $y \in Y$. Then $g^{-1}(y) = \{ v \in Y^k \ | \ g(v) = y \}$. Suppose $v \in g^{-1}(y)$. For $1 \leq i \leq k$, $f_{i}^{-1}(v_i) = \OO_{y,v}^i \cup E_{y,v}^i$. The set of points satisfying $v$ is given by
  \begin{align*}
    \bigcap_{i = 1}^{k}f_{i}^{-1}(v_i) &= \bigcap_{i = 1}^{k}(\OO_{y,v}^i \cup E_{y,v}^i) \\
                                       &= \left(\bigcap_{i = 1}^{k}\OO_{y,v}^i \right) \cup E_{y,v} \\
                                       &= \OO_{y,v} \cup E_{y,v}
  \end{align*}
  where $E_{y, v} = \bigcap_{i = 1}^{k}(\OO_{y,v}^i \cup E_{y,v}^i) \setminus (\bigcap_{i = 1}^{k}\OO_{y,v}^i)$ and $\OO_{y,v} = \bigcap_{i = 1}^{k}\OO_{y,v}^i$.
	Distributing the intersection operator, $E_{y, v} = \bigcup_{i} \bigcap_{j = 1}^{k} C_{ij}$ where for every $i$ there is at least one $j$ such that $C_{ij} = E_{y,v}^i$. Since meagre and null sets are closed under subset and countable union and $E_{y,v}$ is the union of $\TT_{\varphi}$-meagre and $\mu$-null sets, $E_{y, v}$ is $\TT_{\varphi}$-meagre and $\mu$-null.
	Correspondingly, $\OO_{y,v}$ is open in $\TT_{\varphi}$, since it is the finite intersection of $\TT_{\varphi}$-open sets.

  Then we obtain $f^{-1}(y)$ as follows:
  \begin{align*}
    f^{-1}(y) &= \bigcup_{v \in g^{-1}(y)}(\OO_{y,v} \cup E_{y,v}) \\
              &= \left(\bigcup_{v \in g^{-1}(y)}\OO_{y,v} \right) \cup \left(\bigcup_{v \in g^{-1}(y)}E_{y,v} \right) \\
              &= O_y \cup E_y
  \end{align*}
  where $\OO_y = \bigcup_{v \in g^{-1}(y)}\OO_{y,v}$ and $E_y = \bigcup_{v \in g^{-1}(y)}E_{y,v}$. Observe that $\OO_y$ is open in $\TT_{\varphi}$, since it is the union of $\TT_{\varphi}$-open sets, and $E_{y}$ is $\TT_{\varphi}$-meagre and $\mu$-null, since it is the union of $\TT_{\varphi}$-meagre and $\mu$-null sets.
\end{proof}

One may suspect that this notion of explainability is too permissive and designates all classifiers as explainable. In general, this is not the case. Consider the classifier with a single continuous feature that maps rational numbers to $1$ and irrational numbers to $0$. If $\varphi$ is the predicate for open intervals, then for $x \in \RR$ if $x$ is rational (irrational) then every rule satisfying $\varphi$ that covers $x$ contains an irrational (rational) number. Since this holds for each point in the feature space, this classifier is not explainable.

Note though that the above example is qualitatively different from typical classifiers deployed in applications. This leads us to make the following imprecise conjecture:

\begin{conj}(In-the-wild conjecture)
  Classifiers deployed in applications satisfy explainability with respect to a typical explanation scheme.
\end{conj}

\noindent If the in-the-wild conjecture is true, then rule-based explanations and, thus, explanation algorithms based on this formal model are model agnostic.

\section{Implications for Rule-Based Explanation Algorithms} \label{s:choosing}

In this section, we employ the framework developed above to argue for two principles for designing rule-based explanation algorithms.

\begin{prin} \label{prin:1}
  For continuous feature spaces, explanations can take nearly any desired shape.
\end{prin}

The candidate scalable rules discussed thus far - open balls and open rectangles - both generate the Euclidean or standard topology on continuous feature spaces i.e. both scalable rules generate the same explanation topology. However, the standard topology has many more bases. A typical result in topology is to prove that two bases generate the same topology \cite{Mu00}. Scalable rules $\varphi, \psi$ generate equivalent explanation topologies if for every potential explanation $A \subseteq X$ satisfying $\varphi$ and $x \in A$, then there is some potential explanation $B \subseteq X$ satisfying $\psi$ such that $x \in B$ and $B \subseteq A$, and vice versa. 

If two explanation schemes share a coverage measure and their respective scalable rules each generate the same explanation topology, then they share the same class of explainable models. From this perspective, one can substitute any scalable rule that generates the standard topology in place of open balls or open rectangles in an explanation scheme without affecting which models are explainable for the new scheme. Recall, however, that there may be other reasons to prefer some sorts of explanations to others \cite{Mo18}. For instance, users may prefer shorter explanations or explanations that include specific features they care about \cite{watson2021explanation}.

\begin{prin} \label{prin:2}
  If features are unbounded and a probability measure is not known, then the user should only consider scalable rules that are bounded.
\end{prin}

Let us say that a feature is bounded if the feature has a maximum value and a minimum value. Suppose the user does not know a probability measure over a continuous feature space, a single feature $F$ is not bounded, and coverage is a monotone non-decreasing function of Lebesgue measure. Then all rules unbounded in $F$ have equal coverage, namely $\infty$. This is to say that coverage cannot distinguish between the size of unbounded rules; although, some such rules are clearly \emph{larger} than others. This is evident by considering the coverage of such an unbounded rule on the subspace excluding the feature $F$. Since this feature subspace is bounded, it is possible (though not necessary) for coverage to distinguish between rules of various sizes.

This issue can be resolved by restricting rules $\varphi$ to the class of bounded scalable rules. A bounded rule is a rule that includes an upper bound predicate for features with no maximum value and a lower bound predicate for features with no minimum value. Practical examples include open balls with a rational center point and radius and open rectangles with rational corner points.

As mentioned above, one measure of the complexity of a rule is the number of predicates conjoined. While restricting to bounded rules solves the discrimination problem for Lebesgue measure, it introduces a lower bound on the complexity of potential explanations. For instance, each unbounded feature requires an upper bound predicate if unbounded above and a lower bound predicate if unbounded below. In the least, the length of explanations is the number of unbounded features.

\section{Limitations and Open Problems} \label{s:open_probs}

The topological notion of explainability developed in this paper relates the shape of a rule to whether or not a classifier is explainable. A limitation of this approach is that it only considers a single type of explanation, rule-based explanations, and, for rule-based explanations, may not capture a user's preferences over potential explanations. Below, we consider two properties of existing rule-based explanation algorithms and suggest directions for future research.

\subsection{Coverage Guarantee}

For a given explanation scheme $(X, \varphi, \mu)$, suppose a user is only interested in potential explanations with coverage at least $\alpha$. In practice, explanations of sufficiently low coverage may either fail to serve as perturbation resistent explanations in the case where the rule is too narrowly defined or not be relevant to the data distribution in the case where the rule is $\mu$-null.  The inclusion of a coverage guarantee extends explainability in the sense of Definition \ref{def:explainable} which, for an explainable model, guarantees than an explanation exists not that an explanation of a given coverage exists.

One approach to extending topological explainability is to modify the rule $\varphi$ defining potential explanations to depend on $\mu$. Let us define $\varphi_{\mu, \alpha}$ to restrict to potential explanations with coverage greater than $\alpha$ such that the set of potential explanations becomes $\{ A \subseteq X | \varphi(A), \mu(A) \geq \alpha \}$. Observe, however, that $\varphi_{\mu, \alpha}$ is not a scalable rule by failing to satisfy Condition 2 from Definition \ref{def:scalable_rule} and so is not a basis for a topology. Finding a simple way to encode a coverage guarantee without compromising topological structure would be a significant improvement to the present formalism.

\subsection{Fuzzy Explanations}

Existing rule-based explanation algorithms such as Anchors return explanations where some but not necessarily all points satisfying the rule evaluate to the same label. This property results from searching for rules via sampling \cite{Ri18}. Likewise, a user may prefer a fuzzy explanation in cases where one's classifier is complex, perhaps due to overfitting. We term explanations of this sort fuzzy explanations and call the degree to which a fuzzy explanation shares the same label as the point to be explained its fidelity. Let us formalize fidelity as follows:

\begin{defn}
	For explanation scheme $(X, \varphi, \mu)$, the fidelity of explanation $A$ for classifier $f: X \rightarrow Y$ at $x \in X$ is given by $\frac{\mu \left(A \cap X_{f(x)} \right)}{\mu(A)}$ where $X_{f(x)} = \{x' \in X | f(x') = f(x) \}$ if $\mu(A) > 0$ and $0$ otherwise.
\end{defn}

Explainability in the sense of Definition \ref{def:explainable} considers only maximum fidelity rules. Unlike the case with coverage guarantees, extending topological explainability to fuzzy explanations satisfying a given level of fidelity does not compromise the structure of the explanation topology $\TT_{\varphi}$; rather, the characterization of an explainable classifier requires modification. This fuzzy explainability is a strict generalization of the former notion; there are simply more potential explanations for each point in the feature space. The difficulty here is that we can no longer represent explainability as the preimage of each label being an open set excluding a set of edge cases, since potential fuzzy explanations are not contained to the preimage of a single label. In particular, the subset of a potential fuzzy explanation contained in a label preimage can be of arbitrarily high topological complexity.

\section{Related Work} \label{s:related_work}

Rule-based explanation methods are often based on work in rule induction \cite{grzymala2005rule, macha2022rulexai}. Additional rule-based explanation algorithms include EXPLAN \cite{rasouli2020explan}, LIMREF \cite{rajapaksha2022limref}, and XPlainer \cite{ignatiev2019validating}. There are several notions of explainability beyond rule-based explanations. Let us briefly mention three: counterfactuals, surrogate models, and feature importance.

Given a point in the feature space, counterfactuals provide a collection of nearby points that the model labels differently from the given point. Counterfactuals explain a prediction by describing which changes in the feature values would have yielded a different result. \citealt{guidotti2022counterfactual} provides an up-to-date comprehensive survey of counterfactual explanations.

Surrogate models approximate an arbitrary model with an interpretable model. If the approximation is sufficiently close, then interpreting the surrogate model explains the behavior of the arbitrary model. Global approaches seek to learn a single surrogate model, while local approaches learn a surrogate model at a given point in the feature space. For instance, \citealt{bastani2019interpreting} proposes a method for learning a simple decision tree as a global surrogate, and LIME is a well-known technique learing a weighted linear model as a local surrogate \cite{Ri16}.

Feature importance explains model behavior by measuring how much each feature contributes to a prediction. A well-known approach is SHAP which is based on the game-theoretic concept of Shapley values \cite{lundberg2017unified}. Many local surrogate methods such as LIME will also yield some way of ranking features by importance \cite{Ri16}.

Several threads of research have sought to better understand explainability methods. \citealt{mullins2019identifying} is a prior attempt to express rule-based explainability in terms of topology. Adversarial attacks have been used to quantify the degree to which various methods can be manipulated \cite{wilking2022fooling}. Logic-based approaches encode models in a formal language and derive explanations \cite{ignatiev2019abduction}. \citealt{watson2021explanation} introduces an expansive framework to represent explanation methods as a communication game between two learners.

\bibliography{shexp}

\end{document}